\documentclass[letterpaper]{article} 
\usepackage{aaai23}  
\usepackage{times}  
\usepackage{helvet}  
\usepackage{courier}  
\usepackage[hyphens]{url}  
\usepackage{graphicx} 
\usepackage{natbib}  
\usepackage{caption} 
\usepackage{algorithm}
\usepackage{algorithmic}
\usepackage{mathtools}
\usepackage[exponent-product=\cdot,detect-weight=true,detect-family=true,binary-units=true]{siunitx}
\usepackage{booktabs}
\usepackage{multirow}
\usepackage{amsthm}
\newtheorem{theorem}{Theorem}

\nocopyright 

\title{FastAMI --- a Monte Carlo Approach to the Adjustment for Chance\\in Clustering Comparison Metrics}
\author{
    Kai Klede, Leo Schwinn, Dario Zanca, Bj\"orn Eskofier
}
\affiliations{
    Machine Learning and Data Analytics Lab,\\ 
    Friedrich-Alexander Universit\"at Erlangen-N\"urnberg (FAU)\\
    \{kai.klede, leo.schwinn, dario.zanca, bjoern.eskofier\}@fau.de
}

\begin{document}

\maketitle

\begin{abstract}
  Clustering is at the very core of machine learning, and its applications proliferate with the increasing availability of data. However, as datasets grow, comparing clusterings with an adjustment for chance becomes computationally difficult, preventing unbiased ground-truth comparisons and solution selection. We propose FastAMI, a Monte Carlo-based method to efficiently approximate the Adjusted Mutual Information (AMI) and extend it to the Standardized Mutual Information (SMI). The approach is compared with the exact calculation and a recently developed variant of the AMI based on pairwise permutations, using both synthetic and real data. In contrast to the exact calculation our method is fast enough to enable these adjusted information-theoretic comparisons for large datasets while maintaining considerably more accurate results than the pairwise approach.
\end{abstract}

\section{Introduction}

Clustering comparison measures such as the \textit{mutual information} or the \textit{Rand index} are most commonly used to validate clusterings when ground truth information is available \cite{aggarwal_data_2013}. In the context of graphs, the mutual information of the neighborhoods of two nodes indicates their similarity and can predict missing links \cite{hoffman_note_2015,shakibian_mutual_2017}. Another application is multiple-clustering algorithms, which measure the mutual information or one of its variants to identify multiple qualitatively different clustering solutions for a single dataset \cite{muller_discovering_2010, wei_inductive_2021}. Other uses include categorical feature selection, where each feature is understood as a cluster or for the solution selection in consensus clustering \cite{lancichinetti_consensus_2012}.

A well-known problem with these clustering comparison measures is that they do not assume a constant baseline value when comparing two random clustering partitions. Instead, they tend to be larger when the number of clusters approaches the number of data points \cite{vinh_information_2009}, leading to a bias towards smaller clusters when used as an external validation criterion. To obtain a constant baseline, these measures have been adjusted by their expectation value under random permutations of the cluster labels in the \textit{Adjusted Rand index} (ARI) \cite{hubert_comparing_1985} and the \textit{Adjusted Mutual information} (AMI) \cite{vinh_information_2010}.

\citet{romano_adjusting_2016} created an overview of those clustering comparison measures mentioned above and others. They analyzed the relationships between these measures and formulated a guideline for which one to use in which situation. Quoting directly from the abstract, the authors conclude that
\begin{quote}
  ARI should be used when the reference clustering has large equal-sized clusters; AMI should be used when the reference clustering is unbalanced and there exist small clusters.
\end{quote}
However, it has been shown that the computational complexity of the adjusted mutual information is $\mathcal{O}(\max(R,C)N)$, where $R,C$ are the numbers of clusters in the clusterings to be compared and $N$ is the number of data points \cite{romano_standardized_2014}. For large datasets with many small clusters and few bigger clusters, which is exactly the domain where the AMI should be used, its computation becomes difficult and often impractical. Many datasets originating from different domains exhibit these characteristics, such as social networks, collaboration networks, or web graphs \cite{leskovec_snap_2014}. In practice, many authors resort to the non-adjusted normalized mutual information \cite{tian_clustering_2019,chunaev_community_2020,karim_deep_2021}, which is faster to compute, but does not account for random coincidences.

As a workaround, \citet{lazarenko_pairwise_2021} proposed to consider only pairwise permutations in the adjustment for chance, i.e., to only adjust for clusterings where two individual samples exchanged their cluster labels. While this approach allows for faster computations ($\mathcal{O}(R C)$), these pairwise permutations lead to a higher amount of shared information. In Figure \ref{fig:synthetic_emi} we show that the actual \textit{expected mutual information} (EMI) is overestimated by the pairwise EMI. Consequently, the pairwise AMI systematically underestimates the AMI, and the results cannot be compared with existing values. Even the relative order of cluster similarity is affected, as shown in Table \ref{tab:gagolewski_results}.

With FastAMI, we propose a Monte Carlo-based approach to enable AMI comparisons for large datasets with small clusters\footnote{Code at https://github.com/mad-lab-fau/fastami-benchmark}. The presented method allows for fine-grained control over its accuracy and addresses the shortcomings of the pairwise approach. We prove a pessimistic upper bound to the expected runtime that is on par with the pairwise approach. We demonstrate, on synthetic and real data, that FastAMI is considerably faster in practice while producing more accurate results. FastAMI also works when both the exact and the pairwise implementation fail due to memory requirements or impractical runtimes. In the process, we identify and fix a common flaw in synthetic clustering comparison benchmarks.

While the AMI is unbiased when comparing two random clusterings, it has been shown that it still exhibits a bias towards selecting larger clusters when comparing multiple random clusterings against a fixed ground truth \cite{romano_standardized_2014}. As a solution, it was suggested to account also for the variance in the \textit{Standardized Mutual Information} (SMI), which, according to the authors, should be employed when ${N}/{R C} < 5$ and more than three clusterings are compared. However, adoption of the SMI is lagging behind the AMI, and a potential reason might be its steep computational cost $\mathcal{O}(R C N^3)$. We extend our Monte Carlo-based approach to the SMI, making it a computationally feasible alternative for small to medium-sized datasets.

\section{Background and Related Work}
Let $S$ be a set of $N$ data points $\{s_1,s_2,\dots s_N\}$, and two clusterings the surjections $u\colon S \to \{1,2,\dots,R\}$ and $v\colon S \to \{1,2,\dots,C\}$ that partition the dataset $S$ into $R$ and $C$ clusters respectively. We denote with $a_i \coloneqq |u^{-1}(i)|,\; b_j \coloneqq |v^{-1}(j)|$ the cluster sizes and $n_{ij} \coloneqq |u^{-1}(i) \cap v^{-1}(j)|$ the number of shared data points between two clusters. 

The \textit{mutual information} \cite{cover_elements_2008} of the two clusterings is given by
\begin{equation}
  I(u,v) \coloneqq \sum_{i=1}^R\sum_{j=1}^C \frac{n_{ij}}{N} \log{\left(\frac{N \cdot n_{ij}}{a_i \cdot b_j}\right)},\label{eq:mutual_information}
\end{equation}
where ${n_{ij}}/{N} \log{\left({N n_{ij}}/(a_i b_j)\right)}$ is defined to be zero for $n_{ij}=0$. For example, the mutual information of two clusterings that assign every point in the dataset to a single cluster is zero. In the other extreme, when every point is an individual cluster in both clusterings,  the mutual information is maximal $\log{N}$. These simple examples already illustrate a problem for the use as a cluster evaluation metric: The mutual information is expected to increase with the number of clusters (Figure \ref{fig:synthetic_emi}), which induces a bias towards more granular clusterings.

Therefore, the mutual information is adjusted with its expected value under random permutations of the cluster assignments, while the number and size of the clusters $A = \{a_1,\dots a_R\},\; B = \{b_1,\dots b_C\}$ is kept constant (\textit{random permutation model}) \cite{vinh_information_2009}
\begin{align}
  \operatorname{E}\{I\,|\,A,B\} =& \sum_{i=1}^R \sum_{j=1}^C \sum_{\substack{
      n_{ij}=\max(0, \\ a_i + b_j - N)
    }
  }^{\min(a_i,b_j)}\nonumber\\
  &\frac{n_{ij}}{N} \log\left(\frac{N n_{ij}}{a_i b_j}\right) P\{n_{ij}|a_i,b_j,N\}. \label{eq:expected_mutual_information}
\end{align}
Here $P\{n_{ij}|a_i,b_j,N\}$ denotes the probability that $n_{ij}$ of the $a_i$ data points in the cluster $u^{-1}(i)$ also lie in the cluster $v^{-1}(j)$ of size $b_j$ after random permutation and is given by the hypergeometric distribution.

The \textit{adjusted mutual information} can subsequently be defined as
\begin{equation}
  \operatorname{AMI}(u, v)\coloneqq \frac{I(u,v) - \operatorname{E}\{I|A,B\}}{\operatorname{avg}(H(u), H(v)) - \operatorname{E}\{I|A,B\}},\label{eq:definition_ami}
\end{equation}
where $\operatorname{avg}(H(u), H(v))$ is an upper bound to the mutual information given $A,B$ that serves as a normalizer, with the entropy
\begin{equation}
    H(u)\coloneqq\sum_{i=1}^{R}\frac{a_i}{N}\log{\frac{a_i}{N}}.
\end{equation}
The arithmetic and geometric mean, the minimum, and maximum can be used for $\operatorname{avg}$ \cite{vinh_information_2010}. In this work, we choose the arithmetic mean following the default choice of {\tt sklearn} \cite{pedregosa_scikit-learn_2011}.

To reduce the computational complexity of the expected mutual information, \citet{lazarenko_pairwise_2021} restricted its computation to pairwise permutations only
\begin{equation}
  \operatorname{EMI}_\text{pair}(u,v) \coloneqq \operatorname{E}\{I(u,\sigma \circ v)\,|\,\sigma\in S_N: \sigma \text{ pairwise}\}.
\end{equation}
A permutation $\sigma\in S_N$ is pairwise when there exist $i,j\in\{1,\dots,N\}$ such that $\sigma(i) = j$, $\sigma(j)=i$, and $\sigma(t)=t\; \forall t\neq i,j$.
While the original authors defined the pairwise AMI without normalization, we normalize it as in equation \ref{eq:definition_ami} to enable a direct comparison with the exact results.

The $\operatorname{AMI}$ has a constant baseline when comparing two clusterings directly. However, when comparing clusterings via a constant ground truth reference, the $\operatorname{AMI}$ is biased towards larger numbers of clusters \cite{romano_standardized_2014}. The definition of the \textit{standardized mutual information} as
\begin{equation}
  \operatorname{SMI}(u,v) \coloneqq \frac{I(u,v)-\operatorname{E}\{I\,|\,A,B\}}{\sqrt{\operatorname{Var}\{I\,|\,A,B\}}}\label{eq:standardized_mutual_information}
\end{equation}
accounts for that bias. The $\operatorname{SMI}$ indicates by how many standard deviations two clusterings deviate from each other under the hypothesis of random and independent clusterings.

\begin{figure}
  \centering
  \includegraphics[width=0.9\columnwidth]{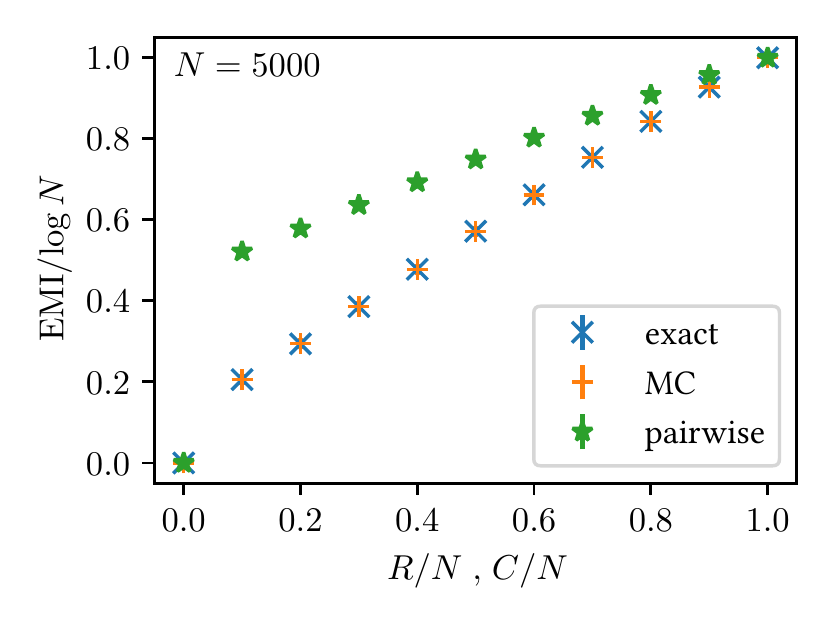}
  \caption{The mutual information is expected to increase as the number of clusters $R=C$ in the clusterings to be compared approaches the number of data points $N$. The expected value under pairwise permutations overestimates the expected mutual information, while our Monte Carlo (MC) approach reproduces the exact results. The figure shows the average values for $200$ pairs of clustering marginals $A, B$, where each marginal is chosen uniformly at random from all partitions of $N=5000$ into $R=C$ parts.}\label{fig:synthetic_emi}
\end{figure}

\section{Method}
The AMI exhibits poor runtimes, mainly because the EMI is computationally expensive. The calculation of $\operatorname{Var}\{I|A,B\}$ for the SMI is even more demanding. Therefore we approximate the EMI and the variance of the mutual information via Monte Carlo sampling.

\subsection{Monte Carlo Estimate of the EMI}
We observe that the shared information of two clusters, i.e., the summands in Equation \ref{eq:expected_mutual_information}, is fully determined by the cluster sizes $a, b$, their overlap $n$ and the total number of samples $N$. Hence, the expected mutual information can be reformulated in terms of the probability $P\{a|A\}$ of a cluster size $a$ in the marginals $A$
\begin{align}
  E\{I|A,B\} = & R C \sum_a \sum_b P\{a|A\} P\{b|B\}\nonumber                                               \\
               & \sum_{n} \frac{n}{N} \log\left(\frac{N n}{a b}\right) P\{n|a,b,N\}. \label{eq:emi_overlap}
\end{align}
We can now proceed to draw samples for $a, b$ and $n$ and approximate the sum through the sample mean. We can readily generate variates for $a$ and $b$ in constant time via Walker's method of alias \cite{walker_new_1974, walker_efficient_1977} after an initial setup cost of $\mathcal{O}(R\log R + C\log C)$ for determining the cluster size frequencies of $A$ and $B$.

Directly drawing the overlap of two clusters $n$ from the hypergeometric distribution would lead to a high occurrence of no overlap $n=0$ in the case of small cluster sizes. However, these terms do not contribute to the mutual information. To exploit that sparsity, we absorb the linear term of the mutual information into the probability density function
\begin{equation}
  n P\{n|a,b,N\} = n\frac{\binom{b}{n} \binom{N-b}{a-n}}{\binom{N}{a}}.
\end{equation}
The binomial coefficients are rewritten as
\begin{equation}
  n\binom{b}{n} = \frac{b!}{(n-1)!\,(b-n)!} = b \binom{b-1}{n-1},
\end{equation}
allowing us to draw variates $n-1$ according to the hypergeometric distribution with $a-1,b-1$ and $N-1$
\begin{align}
  \begin{split}
    n P\{n|a,b,N\} =& \frac{a\cdot b}{N} \frac{\binom{b-1}{n-1} \binom{N-b}{a-n}}{\binom{N-1}{a-1}}\\
    =& \frac{a\cdot b}{N}P\{n-1|a-1,b-1,N-1\}.
  \end{split}\label{eq:hypergeometric_trick}
\end{align}

Several methods to draw hypergeometric variates in on average constant time exist \cite{kachitvichyanukul_computer_1985,hormann_universal_1994}, here we chose a ratio of uniforms approach introduced by \citet{stadlober_ratio_1990}. The approximation of Equation \ref{eq:emi_overlap} terminates, when the relative error estimate of the $\operatorname{EMI}$ reaches a desired precision ${s_{\operatorname{EMI}}}/{\operatorname{EMI}}\leq p$. To ensure termination when the $\operatorname{EMI}$ approaches zero, the algorithm switches to an absolute error criterion $s_{\operatorname{EMI}} \leq p$ when $\operatorname{EMI} < 1$.

Summarizing all the steps above, we formulate Algorithm \ref{alg:emi_fixed_precision}, a Monte Carlo estimate for the EMI. Samples can be drawn in constant time after an initial setup cost of $\mathcal{O}(R\log R + C\log C)$, such that in practice, the number of Monte Carlo samples governs the runtime.

\begin{algorithm}[t]
  \caption{The core of FastAMI is a Monte Carlo approximation to the expected mutual information. We use a numerically stable variant of Welford's online algorithm for calculating the mean and sample variance and define a stopping criterion when the desired precision is reached.}\label{alg:emi_fixed_precision}
  \begin{algorithmic}
    \REQUIRE Cluster sizes $A,B$, EMI precision $p$, minimum number of samples $i_{\text{min}} > 1$
    \ENSURE $\operatorname{EMI}\sim \mathcal{N}\left(E\{I|A,B\}, s_{\operatorname{EMI}}^2\right)$ where $\min\{s_{\operatorname{EMI}}/\operatorname{EMI},s_{\operatorname{EMI}}\} \leq p$
    \STATE $N \gets \sum_{a \in A} a$
    \STATE $i \gets 0,\;\operatorname{EMI} \gets 0,\;M_2 \gets 0$
    \STATE Setup Walker random sampling for $A, B$
    \WHILE{$i < i_{\text{min}}$ or $M_2 > (p \max\{1,\operatorname{EMI}\})^2 \cdot i \cdot(i-1)$}
    \STATE $i \gets i + 1$
    \STATE Draw $a$, $b$ via Walker's Alias method
    \STATE Draw $m \sim P\{m|a-1,b-1,N-1\}$ via Stadlober's method
    \STATE $x \gets a b \log{[{(m + 1) N}/{(ab)}]}$
    \STATE $\Delta_0 \gets x - \operatorname{EMI}$
    \STATE $\operatorname{EMI} \gets \operatorname{EMI} + \Delta_0/i$
    \STATE $\Delta_1 \gets x - \operatorname{EMI}$
    \STATE $M_2 \gets M_2 + \Delta_0 \Delta_1$
    \ENDWHILE
    \STATE $s_{\operatorname{EMI}} = \sqrt{M_2/(i\cdot(i-1))}$
  \end{algorithmic}
\end{algorithm}

\subsection{Upper Bound to the Expected EMI Runtime}\label{sec:expected_runtime_emi}
Given that samples can be drawn on average in constant time, we derive an asymptotic upper bound to the expected number of samples required for convergence of the EMI.

\begin{theorem}
    The expected number of samples required for Algorithm \ref{alg:emi_fixed_precision} with precision $p$ to terminate on two clusterings with $R$ and $C$ clusters and $N\geq 3$ data points is asymptotically bounded by $\mathcal{O}\left({RC}/{p^2}\right)$. \label{thm:asymptotic_upper_bound_emi_runtime}
\end{theorem}
\begin{proof}
    The absolute error of the Monte Carlo approximation $s_{\operatorname{EMI}}$ for a given number of samples $N_\text{samples}$ is
    \begin{equation}
      s^2_{\operatorname{EMI}} = \frac{\operatorname{Var}\left\{\frac{RC}{N^2} a b \log\left(\frac{N n}{a b}\right) | A,B\right\}}{N_\text{samples}},
    \end{equation}
    in the limit of many samples, according to the central limit theorem. Conversely, the expected number of samples to reach a particular error is
    \begin{equation}
      N_\text{samples} = \frac{\operatorname{Var}\left\{\frac{RC}{N^2} a b \log\left(\frac{N n}{a b}\right) | A,B\right\}}{s^2_{\operatorname{EMI}}}.
    \end{equation}
    The Bhatia-Davis inequality \cite{bhatia_better_2000} provides an upper bound to the variance when the variates are bounded $m \leq \frac{RC}{N^2} a b \log\left(\frac{N\cdot n}{a b}\right) \leq M$
    \begin{equation}
      \operatorname{Var}\left\{\frac{RC}{N^2} a b \log\left(\frac{N n}{a b}\right) | A,B\right\} \leq \left(M - \operatorname{EMI}\right)\left(\operatorname{EMI} - m\right).\label{eq:bhatia-davis}
    \end{equation}
    Equality holds when the probability mass function is concentrated on the extremes of the interval. The information contained in the overlap of two clusters is larger or equal than zero ($m=0$), where equality is reached when one of the clusterings has only a single cluster $a=N$ or $b=N$. For the upper bound, note that the overlap of two clusters is maximally the size of the smaller cluster $n\leq \min\{a,b\}$. Without loss of generality, assume $a\leq b$
    \begin{equation}
      ab\log{\frac{nN}{ab}} \leq N b\log{\frac{N}{b}} \leq \frac{N^2}{e} = M \quad \text{for } N\geq e,
    \end{equation}
    With Equation \ref{eq:bhatia-davis} this gives the following upper bound to the expected number of required Monte Carlo samples
    \begin{equation}
      N_\text{samples} \leq \frac{(RC - \operatorname{EMI})\operatorname{EMI}}{s_{\operatorname{EMI}}^2}
    \end{equation}
    The convergence criterion in Algorithm \ref{alg:emi_fixed_precision} requires $s_{\operatorname{EMI}} \leq p$ for $\operatorname{EMI} \leq 1$ and $s_{\operatorname{EMI}} \leq p\operatorname{EMI}$ for $\operatorname{EMI} > 1$, such that
    \begin{equation}
      N_\text{samples} \leq \mathcal{O}\left(\frac{RC}{p^2}\right).
    \end{equation}
\end{proof}

This amounts to the same asymptotic runtime as the pairwise adjusted EMI for a fixed precision $p$. However, the Bhatia-Davis inequality gives only a crude theoretical limit and the experiments will shed more light on the algorithm's runtime.

\subsection{Monte Carlo Estimation of the Variance}\label{sec:method_variance}

We compare two approaches for approximating the variance of the mutual information:
\begin{itemize}
    \item \textit{Separate Monte Carlo} - We split the variance into two terms $\operatorname{Var}\{I\,|\,A,B\} = \operatorname{E}\{I\,|\,A,B\}^2 - \operatorname{E}\{I^2\,|\,A,B\}$, where the first term is approximated as in Algorithm \ref{alg:emi_fixed_precision}. We apply Equation \ref{eq:hypergeometric_trick} and related identities to the explicit formula for $\operatorname{E}\{I^2\,|\,A,B\}$ as given in theorem 1 in \cite{romano_standardized_2014}, and approximate it via an analogous algorithm.
    \item \textit{Direct Monte Carlo} - Instead of generating individual values for $n_{ij}$, we randomly sample the space of all contingency matrices $(n_{ij})_{i\in\{1,\dots,R\}}^{j\in\{1,\dots,C\}}$ at once with a method developed by \citet{patefield_algorithm_1981} and observe the mutual information. The estimators for the EMI and the variance are then simply the sample mean and variance.
\end{itemize}
The direct Monte Carlo approach can also be applied to the EMI. However, the whole contingency matrix would be kept in memory. This is not practical for the approximation of the AMI, since memory footprint was one of the limiting factors we optimized (see Table \ref{tab:snap_results}).

\section{Experiments}\label{sec:experiments}

In the previous section, we introduced a method to approximate the EMI and variance via Monte Carlo sampling. We now proceed to evaluate how quickly these methods converge. Different parameter regimes for $R, C$, and $N$ are explored using synthetic data, and we demonstrate the practical use of the method on real datasets from a wide range of fields.

A Laptop with an Intel i7-10750H with $\SI{32}{\giga\byte}$ RAM was used for the synthetic experiments and the experiments on the \textit{Benchmark Suite for Clustering Algorithms}. The experiments with the larger datasets from the \textit{Stanford Large Network Dataset Collection} were performed on an Intel Xeon E5-2680 v4 system with $\SI{512}{\giga\byte}$ RAM.

\subsection{Synthetic Data}\label{sec:synthetic_experiment}

Previous works usually generated random clusterings by assuming a fixed number of clusters and then assigning each data point to a random cluster uniformly \cite{vinh_information_2009,romano_standardized_2014}. This method overemphasizes balanced clusterings, where all clusters have similar sizes. Choosing a random probability distribution instead of the uniform assignment allows more variance in the cluster sizes \cite{lazarenko_pairwise_2021}. However, both methods do not guarantee the fixed number of clusters they were intended to generate. In the first case, exactly one cluster gets assigned zero data points with a probability of $C\left((C-1)/{C}\right)^N$ and using the inclusion-exclusion principle, the probability of at least one empty cluster is
\begin{align}
  P\{\exists i: i \not\in \operatorname{Im}{v}\} =& \sum_{i=1}^C (-1)^{i-1} \binom{C}{i} \left(\frac{C-i}{C}\right)^N.\label{eq:sampling_flaw}
\end{align}
$\operatorname{Im}{v}$ denotes the image of clustering $v$, i.e. the set of all labels. For a fixed number of data points $N$, the chance of an empty cluster grows with the number of clusters, such that for example for $N=100$ and $C=30$ this probability is already $\approx \SI{66.5}{\percent}$. This chance is even amplified when selecting the probabilities at random since cluster labels can have arbitrarily low probability.

Instead, we propose a maximum entropy approach: Given a fixed number of data points $N$ and fixed numbers of clusters $R, C$, we randomly sample the space of all possible cluster size distributions $A$ and $B$ that fulfill these constraints, i.e., we uniformly sample the integer partitions of $R$ and $C$. We use a rejection sampling method that was outlined by \citet{arratia_probabilistic_2015} and originally developed by \citet{fristedt_structure_1993}. Sampling cluster size distributions is enough for the EMI, which depends only on $A, B$, whereas the SMI depends on two concrete clusterings. In this case, we create an ordered cluster from the marginals and shuffle it randomly, guaranteeing fixed $R,C$.

\begin{figure*}
  \centering
  \includegraphics[width=0.9\textwidth]{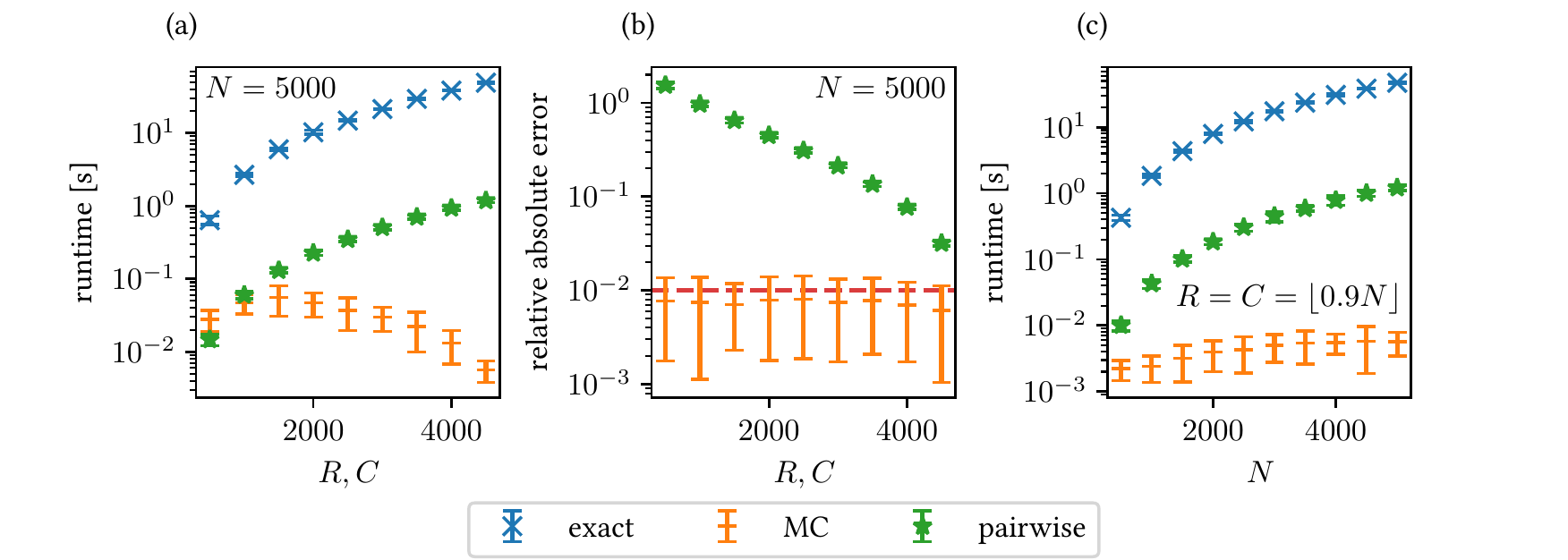}
  \caption{We compute the EMI of 200 pairs of random cluster size distributions with $R=C$ clusters and $N$ data points, comparing three different methods. (a) The exact EMI as implemented in {\tt sklearn} \cite{pedregosa_scikit-learn_2011} and the pairwise EMI take increasingly longer to run as the number of clusters $R=C$ approaches the number of data points $N$. Our Monte Carlo method not only outperforms the latter two, but it performs best for larger numbers of clusters. (b) The Monte Carlo implementation reaches the prescribed precision of $0.01$ (dashed, red line) on average, while the pairwise adjustment introduces model-dependent errors. In (c) we repeat the experiment but now varying $N$ whilst fixing $R=C=\lfloor 0.9 N \rfloor$. The performance gap grows with increasing dataset size.}\label{fig:synthetic_emi_runtime_samples}
\end{figure*}

\subsubsection{EMI}

The computationally demanding part of computing the adjusted mutual information is obtaining the expected value of the mutual information (EMI). Therefore, we compare the methods to compute the EMI for $R=C$ and a fixed $N=5000$. We focus on $R=C$ since the asymptotic runtime for the pairwise and our Monte Carlo algorithm increases the most for that configuration. Both the exact calculation and the pairwise EMI exhibit a nearly $100$-fold increase in runtime, as the number of clusters increases from $500$ to $4500$ (Figure \ref{fig:synthetic_emi_runtime_samples}a). This slowdown is problematic since it is the regime where the AMI is preferable to the Rand Index, as stated by \citet{romano_adjusting_2016}.
On the other hand, our Monte Carlo method performs best in the case of many clusters. The reason is that when smaller clusters dominate as $R=C$ increases, the variance of the number of shared data points $n$ between clusters decreases, and fewer Monte Carlo samples are required for convergence. In the case of only a few larger clusters, there are fewer cluster size values the variates $a$ and $b$ can assume, and hence the runtime peaks somewhere between those extremes. This empirical result confirms that the upper bound in theorem \ref{thm:asymptotic_upper_bound_emi_runtime} overestimates the runtime.

In the regime of many singleton clusters, i.e., high $R,C$, many permutations do not affect the contingency matrix $(n_{ij})$ and hence the mutual information. These permutations contribute to the slowdown of the exact and the pairwise approach, but it also means that leaving them out only slightly affects the pairwise EMI and the relative error is low in that regime (Figure \ref{fig:synthetic_emi_runtime_samples}b). On the other hand, the pairwise approach has high model-dependent errors, where the method is fast. Our approach allows for tunable relative errors across the parameter range (dashed red line in Figure \ref{fig:synthetic_emi_runtime_samples}b).

In a second experiment, we show that the runtime of the exact and pairwise method gets even worse as the number of data points $N$ increases (Figure \ref{fig:synthetic_emi_runtime_samples}c).  

\subsubsection{SMI}

We compare the exact SMI implemented by \citet{romano_standardized_2014} with the separate and direct Monte Carlo approach for estimating the variance. We terminate the calculation when the estimated absolute or relative error estimate of the SMI is below the precision $p=0.1$. Both approximate algorithms outperform the exact SMI, but the separate approach slows down as the number of clusters increases. In contrast, the direct approach provides a speed-up of multiple orders of magnitude across the whole parameter range (Figure \ref{fig:smi_runtime}). A possible explanation could be that the direct approach naturally incorporates the constraint that the variance is larger or equal than zero. The sample variance of the mutual information is always positive, but the difference of two separate estimators for $\operatorname{E}\{I\,|\,A,B\}^2$ and $\operatorname{E}\{I^2\,|\,A,B\}$ can also be negative. In the following we chose the direct approach for FastSMI.

\begin{figure}
    \centering
    \includegraphics[width=0.9\columnwidth]{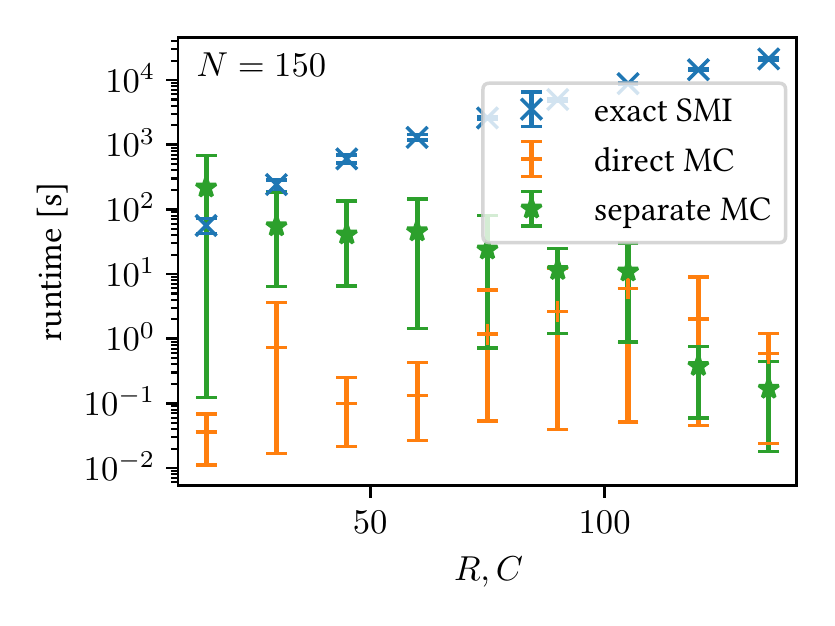}
    \caption{We compute the SMI of $10$ pairs of clusterings with a given number of clusters $R=C$ to a precision of $0.1$. Both Monte Carlo approaches outperform the exact calculation, but the direct approach is multiple orders of magnitude faster in the regime of many smaller clusters.}
    \label{fig:smi_runtime}
\end{figure}

\subsection{Real Data}\label{sec:real_data}

The previous section demonstrated how Monte Carlo approximation speeds up the EMI and SMI for random clusterings. This section shows how that speed-up translates to the practically more relevant AMI, on clusterings that arise from real datasets of various domains.

\citet{lazarenko_pairwise_2021} used 79 datasets from the \textit{Benchmark Suite for Clustering Algorithms} \cite{gagolewski_benchmark_2020} to compare the solutions of several clustering algorithms via the AMI and pairwise AMI. They used the Spearman rank correlation to measure the quality of the results. As a first benchmark, we reproduce these results in Table \ref{tab:gagolewski_results} and extend them to evaluate FastAMI. While our method has a significantly higher Spearman correlation with the exact solution than the pairwise AMI, it was slower compared to the pairwise approach. For the SMI on the other hand, $\SI{61.7}{\percent}$ of the comparisons timed out after $\SI{20}{s}$. FastSMI ($p=0.1$) completed $\SI{99.2}{\percent}$ of the benchmark with a high
Spearman correlation with the exact results, enabling SMI comparisons for medium sized datasets. However, with an average runtime of $\SI{0.013}{\second}$ per comparison, the benchmark suite is not in the domain where an approximation of the AMI is necessary due to runtime constraints.
\begin{table}
    \centering
    \begin{tabular}{lrrr}
    \toprule
    Variant & Result &  Mean &  Spearman \\
    & returned & runtime [s] & (incomplete)\\
    \midrule
sklearn AMI & $\SI{100}{\percent}$ &      $\num{3.6e-3}$ &  - \\
FastAMI & $\SI{100}{\percent}$ &      $\num{1.3e-2}$ &   $\num{1.000}$ \\
pairwise AMI & $\SI{100}{\percent}$ &      $\num{1.6e-3}$ &   $\num{0.577}$ \\
              \midrule
exact SMI & $\SI{38.3}{\percent}$ &      $\num{4.8}$ & - \\
FastSMI & $\SI{99.2}{\percent}$ &      $\num{2.4e-2}$ & $(\num{0.997})$ \\
    \bottomrule
    \end{tabular}
    \caption{Benchmark results on the \textit{Benchmark Suite for Clustering Algorithms} \cite{gagolewski_benchmark_2020} for FastAMI with precision $p=0.01$ and FastSMI with precision $p=0.1$. Both pairwise AMI and FastAMI improve the runtime of the exact version, and FastAMI nearly perfectly matches the relative ranking of the exact result. However, with a maximum $R/N$ of $\num{0.03}$, the Benchmark suite is not in the imbalanced domain where the exact AMI is impractical (Figure \ref{fig:synthetic_emi_runtime_samples}a). The exact SMI timed out after $\SI{20}{\second}$ on $\SI{61.7}{\percent}$ of the $\num{66004}$ comparisons in the Benchmark (column \textit{result returned}). FastSMI returned more results before the timeout in less time per comparison and achieved a high Spearman correlation with the exact results where they were available.}
    \label{tab:gagolewski_results}
\end{table}

Instead, we propose a benchmark based on a collection of large real-world datasets, taken from the \textit{Stanford Large Network Dataset Collection} \cite{leskovec_snap_2014}. We select the datasets designated to community detection from that collection (See Table \ref{tab:snap_results}) and apply the following six methods to find clusterings \cite{staudt_networkit_2016}: 1) Connected Components, 2) Degree Ordered Label Propagation, 3) Label Propagation, 4) Leiden, 5) Louvain, and 6) Louvain Map Equation.

The benchmark then consists of a pairwise comparison of the six clusterings, using AMI, pairwise AMI, and FastAMI. We measure the total runtime and peak memory consumption for the $\num{15}$ comparisons in each dataset and limit every individual comparison to a maximum of $\SI{2000}{\second}$ and $\SI{503.6}{\gibi\byte}$ of memory. If any of these limits is exceeded, we report the respective lower bound in Table \ref{tab:snap_results} and omit the result from further calculations. Following \citet{lazarenko_pairwise_2021}, we report the Spearman rank correlation between the exact results and the respective approximations, excluding the cases where the exact calculation was aborted. Additionally, we report the mean absolute error where exact values are available and substitute it with the according Monte Carlo estimate otherwise.

The AMI is preferable to the ARI when the clusterings are imbalanced \cite{romano_adjusting_2016}. In the synthetic experiments, we only fixed the number of clusters without explicitly specifying the imbalance. For the real datasets, we measure the balance using the mean normalized entropy
\begin{equation}
  \operatorname{Balance}(U) = \sum_{u_k\in U} \frac{H(u_k)}{|U|\log{|\operatorname{Im}u_k|}},
\end{equation}
where $U$ denotes the set of clustering solutions, obtained via the six different methods for the same dataset. A value of one indicates that all the clusterings are perfectly balanced, i.e., every cluster $i$ within a clustering $u$ has the same size. A value of zero on the other hand means that all clusterings consist of a single large cluster and are thus maximally imbalanced.

As expected from the synthetic experiments, our method outperforms the others in terms of runtime. FastAMI also has a much lower memory footprint than the {\tt sklearn} implementation and the pairwise AMI, since it does not keep the (sparse) contingency table in memory.
While FastAMI does slow down for datasets with more nodes $N$, it does not struggle with imbalanced datasets in contrast to {\tt sklearn} and the pairwise AMI. FastAMI performs comparably to the pairwise adjustment regarding the Spearman correlation. The real benefit in terms of quality is that the FastAMI results can directly be compared with the traditional AMI as demonstrated by the mean absolute error, whereas the pairwise AMI is systematically lower and somewhat harder to interpret.

\begin{table*}[t]
      \begin{tabular}{llllrrrrr}
      \toprule
      & & & & Result & Total & Peak & MAE & Spearman \\
      Dataset & Nodes  & Balance & Method & returned & time [$\si{\second}$] & memory & (est.) & (incompl.) \\
      \midrule
      \multirow{3}{*}{Email} & \multirow{3}{*}{\num{1.0e+03}}  & \multirow{3}{*}{\num{3.5e-02}} & sklearn &                  15/15 &              $\num{0.6}$ &      $\SI{208}{\kibi\byte}$ &                - &                     - \\
                 &               &                & pairwise &                  15/15 &             $\textbf{\num{0.4}}$ &      $\SI{184}{\kibi\byte}$ &    $\num{0.357}$ &         $\num{0.943}$ \\
                 &               &                & FastAMI &                  15/15 &              $\num{0.6}$ &      $\textbf{\SI{164}{\kibi\byte}}$ &    $\textbf{\num{0.001}}$ &         $\textbf{\num{0.989}}$ \\
      \cmidrule{4-9}
      \multirow{3}{*}{DBLP} & \multirow{3}{*}{\num{3.2e+05}}  & \multirow{3}{*}{\num{4.3e-03}} & sklearn &                  13/15 &         $> \num{7539.9}$ &                           - &                - &                     - \\
                 &               &                & pairwise &                  15/15 &            $\num{361.0}$ &     $\SI{89.6}{\gibi\byte}$ &  $(\num{0.310})$ &       $(\num{0.972})$ \\
                 &               &                & FastAMI &                  15/15 &              $\textbf{\num{4.5}}$ &    $\textbf{\SI{106.4}{\mebi\byte}}$ &  $(\textbf{\num{0.003}})$ &       $(\textbf{\num{0.986}})$ \\
      \cmidrule{4-9}
      \multirow{3}{*}{Amazon} & \multirow{3}{*}{\num{3.3e+05}} & \multirow{3}{*}{\num{3.8e-03}} & sklearn &                  13/15 &         $> \num{8988.4}$ &                           - &                - &                     - \\
                 &               &               & pairwise &                  15/15 &            $\num{329.9}$ &     $\SI{57.6}{\gibi\byte}$ &  $(\num{0.421})$ &       $(\num{0.966})$ \\
                 &               &               & FastAMI &                  15/15 &              $\textbf{\num{4.5}}$ &     $\textbf{\SI{26.8}{\mebi\byte}}$ &  $(\textbf{\num{0.005}})$ &       $(\textbf{\num{0.972}})$ \\
      \cmidrule{4-9}
      \multirow{3}{*}{Youtube} & \multirow{3}{*}{\num{1.1e+06}} & \multirow{3}{*}{\num{2.4e-04}} & sklearn &                   9/15 &        $> \num{15967.7}$ &                           - &                - &                     - \\
                 &               &               & pairwise &                  13/15 &                        - &  $> \SI{503.6}{\gibi\byte}$ &  $(\num{0.260})$ &       $(\num{0.913})$ \\
                 &               &               & FastAMI &                  \textbf{15/15} &              $\textbf{\num{7.4}}$ &    $\textbf{\SI{119.4}{\mebi\byte}}$ &  $(\textbf{\num{0.005}})$ &       $(\textbf{\num{0.957}})$ \\
      \cmidrule{4-9}
      \multirow{3}{*}{Wikipedia} & \multirow{3}{*}{\num{1.8e+06}} & \multirow{3}{*}{\num{2.6e-02}} & sklearn &                  15/15 &            $\num{714.3}$ &     $\SI{12.0}{\gibi\byte}$ &                - &                     - \\
                 &               &               & pairwise &                  15/15 &             $\num{32.7}$ &      $\SI{8.4}{\gibi\byte}$ &    $\num{0.210}$ &         $\textbf{\num{0.975}}$ \\
                 &               &               & FastAMI &                  15/15 &              $\textbf{\num{7.8}}$ &    $\textbf{\SI{140.7}{\mebi\byte}}$ &    $\textbf{\num{0.001}}$ &         $\num{0.973}$ \\
      \cmidrule{4-9}
      \multirow{3}{*}{Orkut} & \multirow{3}{*}{\num{3.1e+06}}  & \multirow{3}{*}{\num{2.7e-02}} & sklearn &                  15/15 &           $\num{2072.0}$ &     $\SI{27.3}{\gibi\byte}$ &                - &                     - \\
                 &               &               & pairwise &                  15/15 &             $\num{69.9}$ &     $\SI{19.2}{\gibi\byte}$ &    $\num{0.276}$ &         $\num{0.967}$ \\
                 &               &               & FastAMI &                  15/15 &             $\textbf{\num{15.4}}$ &    $\textbf{\SI{341.0}{\mebi\byte}}$ &    $\textbf{\num{0.001}}$ &         $\textbf{\num{0.991}}$ \\
      \cmidrule{4-9}
      \multirow{3}{*}{LiveJournal} & \multirow{3}{*}{\num{4.0e+06}} & \multirow{3}{*}{\num{2.7e-04}} & sklearn &                  10/15 &        $> \num{14496.1}$ &                           - &                - &                     - \\
                 &               &               & pairwise &                  13/15 &                        - &  $> \SI{503.6}{\gibi\byte}$ &  $(\num{0.149})$ &       $(\num{0.937})$ \\
                 &               &               & FastAMI &                  \textbf{15/15} &             $\textbf{\num{23.9}}$ &    $\textbf{\SI{340.8}{\mebi\byte}}$ &  $(\textbf{\num{0.003}})$ &       $(\num{0.937})$ \\
      \cmidrule{4-9}
      \multirow{3}{*}{Friendster} & \multirow{3}{*}{\num{6.6e+07}} & \multirow{3}{*}{\num{1.3e-05}} & sklearn &                    0/15 &                        - &  $> \SI{503.6}{\gibi\byte}$ &                - &                     - \\
                 &               &               & pairwise &                   6/15 &                        - &  $> \SI{503.6}{\gibi\byte}$ &  $(\num{0.031})$ &         - \\
                 &               &               & FastAMI &                  \textbf{15/15} &            $\textbf{\num{439.3}}$ &      $\textbf{\SI{4.6}{\gibi\byte}}$ &  $(\textbf{\num{0.003}})$ &         - \\
      \bottomrule
      \end{tabular}
      \caption{Benchmark of the AMI ({\tt sklearn}), the pairwise AMI and FastAMI, regarding time and memory complexity. We evaluated these methods on the results of six different community detection algorithms for large graph datasets, taken from the Stanford Large Network Dataset Collection \cite{leskovec_snap_2014}. In the column \textit{result returned}, we report how often an algorithm successfully terminated and how often it exceeded the time or memory limit of $\SI{2000}{s}$ and $\SI{503.6}{\gibi\byte}$. Our method outperforms {\tt sklearn} and Lazarenko's pairwise algorithm both in the runtime and the memory footprint, enabling AMI comparisons for dataset sizes that were previously inaccessible. While the pairwise AMI is comparable in terms of Spearman correlation, our method has the additional benefit of tunable absolute errors, allowing for a direct comparison with the exact metric. The mean absolute error is given as Monte Carlo estimate and the Spearman correlation was calculated on a subset when the exact solution was not available.}
    \label{tab:snap_results}
  \end{table*}

\section{Conclusion}

This paper presents an effective and practical method for approximating the Adjusted Mutual Information. FastAMI takes advantage of the sparsity and low variance of contingency tables of clusterings, enabling AMI comparisons on datasets that were computationally inaccessible before. We analyzed the behavior of our approximation scheme based on synthetic data and demonstrated state-of-the-art performance on a suite of real datasets. In future work, one could readily parallelize Algorithm \ref{alg:emi_fixed_precision} to improve its performance further. 
We extended our algorithm to the standardized mutual information, rendering variance-adjusted clustering comparisons computationally feasible for moderately sized datasets.

\appendix
\section*{Acknowledgments}
This work was funded by the \textit{Bayerischen Verbundförderprogramm (BayVFP) – Förderlinie Digitalisierung – Förderbereich Informations- und Kommunikationstechnik} of the Bavarian Ministry of Economic Affairs, Regional Development and Energy and supported by \textit{Bayern Innovativ – Bayerische Gesellschaft für Innovation und Wissenstransfer mbH}.

\bibliography{bibliography}

\end{document}